%% file: TechnicalReport.tex
\newif\ifreport\reporttrue
\documentclass[sigconf]{acmart}
\usepackage{subcaption}
\usepackage{color}
\usepackage[linesnumbered,ruled,vlined]{algorithm2e}
\usepackage{algpseudocode}

\usepackage{tcolorbox}
\usepackage{lipsum}
\usepackage{enumitem}
\AtBeginDocument{%
  \providecommand\BibTeX{{%
    \normalfont B\kern-0.5em{\scshape i\kern-0.25em b}\kern-0.8em\TeX}}}
 \begin{document}
 
 \setcopyright{acmcopyright}
\copyrightyear{2022}
\acmYear{2022}
\acmDOI{XXXXXXX.XXXXXXX}

\acmConference[ACM MobiHoc '22]{The Twenty-third ACM International Symposium on Theory, Algorithmic Foundations, and Protocol Design for Mobile Networks and Mobile Computing}{17-22, Oct. 2022}{Seoul, South Korea}
\acmPrice{15.00}
\acmISBN{978-1-4503-XXXX-X/18/06}

 \title{Online Learning of Whittle Indices for Restless Bandits with Non-Stationary Transition Kernels}

\author{Md Kamran Chowdhury Shisher, Vishrant Tripathi, Mung Chiang, Christopher G. Brinton}
\email{Email: {mshisher,tripathv, chiang, cgb} @purdue.edu}
\affiliation{%
  \institution{Department of ECE, Purdue University}
  \city{West Lafayette}
  \state{Indiana}
  \country{USA} 
}

\settopmatter{printacmref=false}
\newcommand{\ignore}[1]{{}}
\renewcommand\footnotetextcopyrightpermission[1]{}
\pagestyle{plain}

\begin{abstract}
The restless multi-armed bandit (RMAB) framework is a popular approach to solving resource allocation problems in networked systems. 
In this paper, we study optimal resource allocation in RMABs facing unknown and non-stationary dynamics. Solving RMABs optimally is known to be PSPACE-hard even with full knowledge of model parameters. While Whittle index policies offer asymptotic optimality with low computational cost, they require access to stationary transition kernels, an unrealistic assumption in many modern networking applications. To address this challenge, we propose a Sliding-Window Online Whittle (SW-Whittle) policy that remains computationally efficient while adapting to time-varying kernels. Through theoretical analysis, we show that our algorithm achieves sub-linear dynamic regret with respect to the number of episodes. We further address the important case where the variation budget is unknown in advance by combining a Bandit-over-Bandit framework with our sliding-window design. In our scheme, window lengths are tuned online as a function of the estimated variation, while Whittle indices are computed via an upper-confidence-bound of the estimated transition kernels and a bilinear optimization routine. Numerical experiments demonstrate that our algorithm consistently outperforms baselines, achieving the lowest cumulative regret across a range of non-stationary environments.
\end{abstract}


\maketitle

\input{introduction.tex}
\input{relatedworks}

\input{problemsettings}
\input{result}

\input{application}

\input{simulations}
\bibliographystyle{unsrt}
\bibliography{refshisher}

\section{Appendix}
\input{appendix}

\end{document}

%% file: introduction.tex
\section{Introduction} 
Restless multi-armed bandits (RMABs) have received significant attention in solving many sequential decision-making
problems \cite{whittle1988restless, borkar2017opportunistic, Kadota2018, Tripathi2019, shishertimely, Kadota2019, liu2011indexability, ruiz2020multi, dahiya2022scalable, dusonchet2003continuous, meshram2017hidden, meshram2018whittle, villar2015multi, bhattacharya2018restless, lee2019optimal, mate2020collapsing, behari2024decision}. In the RMAB framework, a decision-maker must select a subset of arms to activate at each time step, subject to a global resource constraint, while the other arms remain passive. Each arm is modeled as a Markov decision process, and evolves stochastically according to two different transition kernels, depending on whether the arm is activated or not: the activated arms evolve according to their active Markov transition kernels, while the rest of the arms evolve according to their passive Markov transition kernels. At the end of the time-slot, the decision maker receives rewards from each arm, where rewards are functions of the current state and action.

RMABs have a long history in resource allocation and operations research literature, starting with Whittle's seminal work \cite{whittle1988restless}. Whittle introduced a heuristic policy for the RMAB problem, known as the Whittle index policy. This policy relies on establishing a special mathematical property called \textit{indexability} for each arm and then deriving index functions that map states to the value of activating an arm in that state. At each decision time, the policy activates the $M$ bandits with the highest Whittle indices out of the $N$ bandits. 
The Whittle index achieves asymptotic optimality, provided the RMAB is indexable and has a global attractor point \cite{weber1990index,verloop2016asymptotically, gast:hal-03041176, gast2021lp}.

Over the past four decades, RMABs have been used to model and optimize resource allocation problems in many domains such as wireless scheduling \cite{borkar2017opportunistic, Kadota2018, Tripathi2019, shishertimely, Kadota2019}, machine monitoring and control \cite{liu2011indexability, ruiz2020multi, dahiya2022scalable}, server scheduling \cite{dusonchet2003continuous}, recommendation systems \cite{meshram2017hidden, meshram2018whittle}, and healthcare \cite{villar2015multi, bhattacharya2018restless, lee2019optimal, mate2020collapsing, behari2024decision}. However, most of these prior works utilizing Whittle index-based policies focus on known and stationary transition kernels \cite{NIPS2015_6d70cb65, Tripathi2019, shishertimely, le2008multi, meshram2018whittle}. In each of these aforementioned applications, transition kernels can be unknown and non-stationary, i.e., the laws governing the evolution of states can drift over time. For example:
\begin{itemize}[leftmargin=5mm]
    \item Consider a wireless scheduling problem where a base station must allocate limited frequency channels to a set of mobile users. The achievable data rate or the successful transmission between a user and the base station depends on the channel state. This channel state evolution is typically random and time-varying due to user mobility and shadowing. Thus, the wireless scheduling problem can be formulated as an RMAB, but with unknown and non-stationary transition kernels.
    
    \item Consider a load balancing problem where jobs arrive into a datacenter and a decision-maker assigns jobs to servers. The time required to finish a job at any server depends on its current load and how it evolves over time. This evolution is random and time-varying since there are multiple load balancers and job streams contributing to the load at any given server within a large datacenter. Deciding which server to pick can then be formulated as a non-stationary RMAB.
\end{itemize}


In such settings where the transition kernels of a RMAB are unknown and non-stationary, Whittle indexing becomes an online/reinforcement learning problem. Many papers have designed algorithms for Markov decision processes (MDPs) and MABs with unknown and non-stationary transition kernels and have analyzed dynamic regret \cite{ortner2020variational, cheung2020reinforcement, marin2024metacurl, wei2023provably, wei2021non}. However, these works can not be applied to RMAB due its special structure: \textit{the passive arms (arms that are not activated) continue to evolve stochastically}. Due to this and the combinatorial action space, even when the transition kernels are known, developing an optimal policy for RMABs is well known to be PSPACE-hard \cite{papadimitriou1994complexity}.

On the other hand, applying traditional online learning and reinforcement learning policies \cite{ortner2020variational, cheung2020reinforcement, marin2024metacurl, wei2023provably, wei2021non} naively to these RMABs may lead to inefficient learning performance and to exponential regret bounds. This motivates combining Whittle with online learning methods. In this direction, a recent work \cite{wang2023optimistic} designed a Whittle index-based policy called {\it UCWhittle} for unknown but stationary transition kernels. Although techniques exist for adapting reinforcement learning algorithms to non-stationary environments \cite{wei2021non}, they are not directly compatible with the recently proposed UCWhittle policy \cite{wang2023optimistic}. This incompatibility arises from the unique structure of RMABs and the specific method used to compute the Whittle index via Lagrangian relaxation. 

Motivated by this, 
we first pose the following research question:
\begin{quote}
\textbf{\textit{Can we develop a provably efficient Whittle index-based online algorithm for RMABs facing non-stationary transition kernels?}}  
\end{quote}

Furthermore, while the transition kernels themselves are often unknown, in many real-world applications, prior knowledge regarding the \textit{sparsity} of these kernels may be available. For example, consider a wireless scheduling problem that aims to maximize information freshness by selecting which users (arms) to schedule. In this case, the state can be modeled using Age of Information (AoI) \cite{kaul2012real, YinUpdateInfocom} – a widely used metric for quantifying information freshness. Then, the AoI of an arm increases by one if the arm is not scheduled for transmission. Conversely, if the arm is scheduled, its AoI resets to one with the success probability of the transmission. Thus, the AoI will never increase by 2 or decrease to a value other than 1. Motivated by this example, we also consider the use of the sparsity information in the Whittle index-based online algorithm.

\subsection{Outline and Summary of Contributions}
\begin{itemize}[leftmargin=4mm]
 \item {\bf Algorithm Design.} We design a sliding window-based online Whittle index policy to provide a computationally efficient and adaptive solution for non-stationary RMABs (see Algorithm \ref{alg:UCWWhittleknown} and Algorithm \ref{alg:UCWWhittleunknown}). We model the transition kernel non-stationarity of arm $n$ using a total variation budget $V_n$ which is an upper bound of the sum of the total variational distance. In Algorithm \ref{alg:UCWWhittleknown}, we consider the total variation budget to be known. The Whittle index is predicted via a sliding window and upper confidence bound (UCB) approaches. Moreover, our algorithm takes into account the sparsity of the transition kernels, which significantly simplifies the complexity of optimization. In Algorithm \ref{alg:UCWWhittleunknown}, we extend to the case where the total variation budget is unknown. To estimate the budget $V_n$, we utilize a ``Bandit-over-Bandit'' approach in which $V_n$ is selected from a finite set of possible values.

\item {\bf Dynamic Regret Analysis.} We rigorously characterize an upper bound on the dynamic regret of our algorithms. Given total variation budget $V$, we show that Algorithm \ref{alg:UCWWhittleknown} can achieve sub-linear dynamic regret $\tilde O(T^{2/3}\tilde V^{1/3})$ with number of episodes $T$ (see Theorems \ref{theorem1}-\ref{theorem2}). In the case of an unknown variation budget, we show that Algorithm \ref{alg:UCWWhittleunknown} achieves dynamic regret of $\tilde O(T^{2/3}( V+2T/J)^{1/3})+\tilde O(\sqrt{TJ})$, where $J$ is the number of possible values for the variation budget (see Theorem \ref{theorem4}). It is difficult to analyze dynamic regret of an online policy under non-stationary environments, and even more difficult for RMABs; to the best of our knowledge, our paper is the first to provide dynamic regret for online learning of Whittle indices under non-stationary transition kernels.

\item{\bf Wireless Scheduling Application.}
In Section \ref{wirelessscheduling}, we demonstrate the application of our algorithm to an online wireless scheduling problem that aims to optimize functions of the AoI. Our approach involves selecting transition probabilities that maximize the value function within a specified confidence bound. We show that by leveraging the inherent sparsity of wireless scheduling, we can derive a closed-form solution for these probabilities, significantly reducing computational overhead.

\item {\bf Simulation Results.} 
We demonstrate the performance of our proposed policy by evaluating it under two applications: wireless scheduling
and one-dimensional monotonic bandits modeled as an RMAB. Our simulation results (see Table \ref{tab:performance} \& Fig. \ref{fig:result}) show that our algorithm achieves lower regret in practice compared with several baselines: the UCWhittle policy \cite{wang2023optimistic}, WIQL policy \cite{biswas2021learn}, and a uniformly randomized policy \cite{Kadota2018}.
\end{itemize}

%% file: relatedworks.tex
\subsection{Related Works}\label{RelatedWorks}
As discussed, Whittle's seminal work \cite{whittle1988restless} introduced a heuristic policy for the infinite-horizon RMAB problem, known as the Whittle index policy. 
Many subsequent works have applied the Whittle index framework to different resource allocation problems \cite{NIPS2015_6d70cb65, Tripathi2019, shishertimely, le2008multi, meshram2018whittle, Kadota2018, Kadota2019, ornee2023whittle} by modeling them as RMABs.

When the transition kernels are unknown, a decision maker needs to use an online learning method to compute Whittle indices. Multiple works \cite{avrachenkov2022whittle,fu2019towards, biswas2021learn} have proposed Q-learning algorithms to compute the Whittle Index. The authors in \cite{nakhleh2021neurwin} proposed NeurWIN, and \cite{nakhleh2022deeptop} proposed DeepTOP to compute the Whittle index using neural networks. While these prior works \cite{nakhleh2021neurwin, nakhleh2022deeptop} were not focused on providing regret guarantees for their policy, other works \cite{tripathi2021online, wang2023optimistic, xiong2022reinforcement, wang2020restless, dai2011non, jung2019regret} proposed online learning algorithms for RMABs with static regret analysis. Specifically, the authors in \cite{dai2011non} analyzed regret against an optimal policy from a finite number of potential policies. The paper \cite{jung2019regret} used a Thompson sampling–based algorithm and analyzed Bayesian regret bound under a given prior distribution. \cite{wang2020restless} assumes a policy oracle and develops a meta-learning algorithm. The algorithm in \cite{xiong2022reinforcement} leverages access to an offline simulator to generate samples for any given state–action pair. In \cite{tripathi2021online}, the authors developed an online Whittle algorithm with a static regret guarantee compared to the best fixed Whittle index policy for time-varying cost function.

In \cite{wang2023optimistic}, the authors developed UCWhittle, a UCB-based online learning algorithm for the Whittle index with unknown transition kernels. Evaluating regret with respect to the optimal policy requires computing the optimal solution to the RMAB problem, which is intractable due to the combinatorial nature of the space and action spaces. To overcome this difficulty, \cite{wang2023optimistic} conducted regred analysis of UCWhittle using the Lagrangian relaxed form of the RMAB problem. Following \cite{wang2023optimistic}, we also employ the Lagrangian relaxed problem to evaluate online learning performance. We extend the unknown but stationary transition setting of \cite{wang2023optimistic} to \textit{non-stationary} transition dynamics. To the best of our knowledge, this is the first work to provide dynamic regret analysis of an online Whittle index-based policy for RMABs with non-stationary transitions. \\

\noindent Note that, due to page limitations, we defer detailed proofs of our results to an online technical report \cite{shisher2025Technical}.


%% file: problemsettings.tex
\section{Online Problem Setting}

We consider an episodic RMAB problem with $N$ arms and an unknown non-stationary environment. Each arm $n \in [N]$ is associated with a unichain MDP denoted by a tuple $(\mathcal S, \mathcal A, P_{n, t}, r_n)$ at every episode $t$, where the state space $\mathcal S$ is finite, $\mathcal A=\{0, 1\}$ is a set of binary actions, $P_{n,t}: \mathcal S \times \mathcal A \times \mathcal S \mapsto [0, 1]$ is the transition
kernel of arm $n$ with $P_{n,t}(s^{\prime}|s, a)$ being the probability of transitioning to state $s^{\prime}$ from state $s$ by taking action $a$ in episode $t$, and $r_n(s, a)$ is the reward function for arm $n$ when the current state is $s$ and the action $a$ is taken. The total number of episodes is $T$ and each episode itself consists of $H$ time slots. 
We consider that the transition kernels $P_{n, t}$ are unknown and non-stationary, i.e., $P_{n, t}$ can change across episodes $t \in [T]$.

A decision maker (DM) determines what action to apply to each arm at a decision time $h\in [H]$ of an episode $t\in[T]$ under the instantaneous activation constraint that at most $M$ arms can be activated. Let $s_{n, h, t} \in \mathcal S$ be the state of arm $n$ at time $h$ of episode $t$ and $a_{n, h, t}\in \mathcal A$ be the action taken by the DM for arm $n$ at decision time slot $h$ of episode $t$. If DM decides action $a_{n, h, t}=1$, then the arm $n$ at time slot $h$ of episode $t$ is activated; otherwise, if action $a_{n, h, t}=0$, then the arm $n$ at time slot $h$ of episode $t$ is not activated.

The action taken by the DM in episode $t$ is described by a policy  $\pi_{t}: \mathcal S^{N} \mapsto \mathcal A^N$ which maps a given state $(s_1, s_2, \ldots, s_N) \in \mathcal S^{N}$ to an action $(a_1, a_2, \ldots, a_N) \in \mathcal A^N$. The corresponding expected discounted sum of rewards in episode $t$ is given by 
\begin{align}
    \!\! R_t\left(\pi_t, (P_{n,t})_{n=1}^N\right)\!=\!\!\mathbb E\left[\sum_{h=1}^H \sum_{n=1}^N  \gamma^{h-1} r_n(s_{n, h, t}, a_{n, h, t})\bigg| \pi_{t}, (P_{n,t})_{n=1}^N\right],
\end{align}
where $\gamma$ is the discount factor and $R_t\left(\pi_t, (P_{n,t})_{n=1}^N\right)$ is the expected discount sum of the rewards in episode $t$. The DM aims to maximize the expected discount sum of the rewards across all episodes, subject to arm activation constraints, i.e., 
\begin{align}\label{mainProblem}
    &\max_{\pi_t \in \Pi}R_t\left(\pi_t, (P_{n,t})_{n=1}^N\right);~~\mathrm{s.t.} \sum_{n=1}^N a_{n, h, t} \leq M, \forall h \in [H],\forall t \in [T]
\end{align}
where $\Pi$ is the set of all causal policy $\pi_t: \mathcal S^N \mapsto \{0, 1\}^N$. 

In our online learning setting, we consider that a decision maker can observe the empirical state, action, next-state $(s_{n,h,t-1}, a_{n, h, t-1}$ $,s_{n, h+1, t-1})$ transition history up to episode $t-1$ and adaptively update it's policy $\pi_t$ for the episode $t$, as illustrated in Figure \ref{fig:system}. In this paper, our goal is to develop the online policy update algorithm such that we can achieve sub-linear dynamic regret with the number of total episodes $T$. 

\subsection{Lagrangian Relaxation}
Because the main problem described in \eqref{mainProblem} is intractable, we relax the per-time slot constraint and use the Lagrangian defined below:

\begin{align}\label{Lagrangian}
  &\mathcal L(\pi_t, (P_{n, t})_{n=1}^N, \lambda)\nonumber\\  &:=\mathbb E\left[\sum_{h=1}^H \bigg(\sum_{n=1}^N  \gamma^{h-1} \bigg(r_n(s_{n, h, t}, a_{n, h, t})-\lambda a_{n, h,t}\bigg)\bigg)+\lambda M\bigg| \pi_{t}, (P_{n,t})_{n=1}^N\right],
\end{align}
where $\lambda\geq 0$ is a Lagrangian penalty that is interpreted as the cost to pay for activation.

\begin{figure}[t]
\centering
\includegraphics[width=0.5\textwidth]{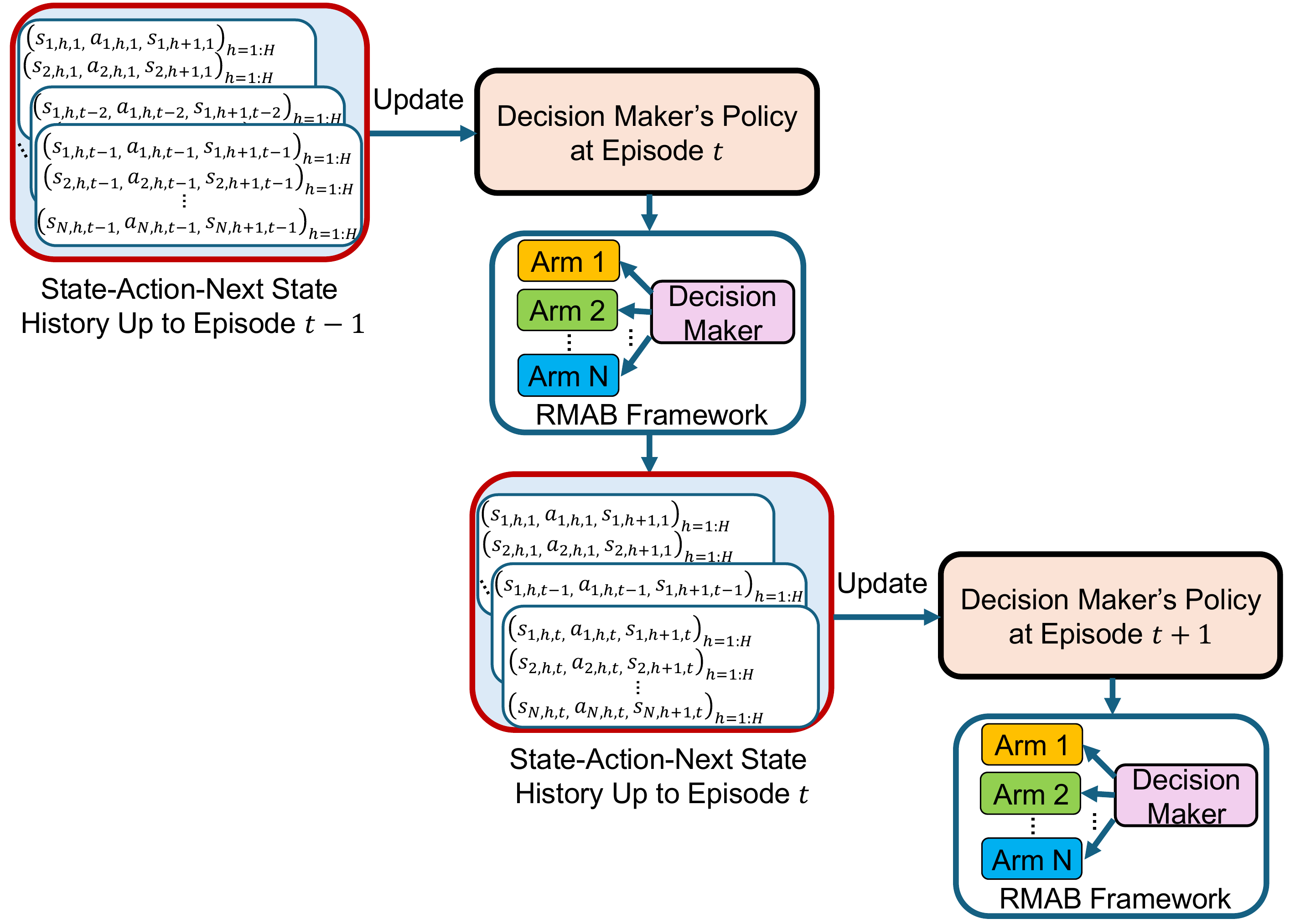}
\caption{\small Online learning setting for RMAB framework:  Decision maker can observe the empirical state, action, next-state transition history up to episode $t-1$ to update policy for the episode $t$. \label{fig:system}
}
\vspace{-3mm}
\end{figure} 

The Lagrangian problem described in \eqref{Lagrangian} enables us to decompose the combinatorial decision problem \eqref{mainProblem} into a set of $N$ independent Markov decision processes for each arm:
\begin{align}\label{decoupledLagrangian}  
   &\max_{\pi_{n, t}\in \Pi_n} U(\pi_{n, t}, P_{n,t}, \lambda),
\end{align}
where $U(\pi_{n, t}, P_{n,t}, \lambda)$ is defined as 
\begin{align}\label{decoupledLagrangianfixpolicy}
   &U(\pi_{n, t}, P_{n,t}, \lambda)=\mathbb E\left[\sum_{h=1}^H \gamma^{h-1} \bigg(r_n(s_{n, h, t}, a_{n, h, t})-\lambda a_{n, h,t}\bigg)\bigg| \pi_{n,t}, P_{n,t}\right].
\end{align}

\subsection{Whittle Index Policy}\label{whittlediscussion}
Given $\lambda$, we denote by $\phi_n(\lambda)$ the set of states for which it is optimal not to activate the arm. The set $\phi_n(\lambda)$ is given by $\phi_n(\lambda):=\{s\in \mathcal S: Q_{n, \lambda, t}(s, 0)> Q_{n, \lambda, t}(s, 1)\}$, where the action value function $Q_{n, \lambda, t}(s, a)$ associated with Bellman optimality equation for \eqref{decoupledLagrangian} is 
\begin{align}
    Q_{n, \lambda, t}(s, a)=r_n(s, a)-\lambda a+ \gamma \sum_{s' \in \mathcal S} P_{n, t}(s'|s, a) V_{n, \lambda, t}(s')
\end{align}
and the value function $V_{n, \lambda, t}(s)$ associated with Bellman optimality equation for \eqref{decoupledLagrangian} is 
\begin{align}
    V_{n, \lambda, t}(s)=\max_{a \in \mathcal A}Q_{n, \lambda, t}(s, a).
\end{align}
Intuitively, as the Lagrangian cost $\lambda$ increases, it is less likely the optimal policy activates arm $n$ in a given state. Hence, the set $\phi_n(\lambda)$ would increase monotonically. 

\begin{definition}[\textbf{Indexability}]
    An arm is said to be indexable if the set $\phi_n(\lambda)$ increases monotonically as $\lambda$ increases from $0$ to $\infty$. A restless bandit problem is said to be indexable if all arms are indexable.
\end{definition}

\begin{definition}[\textbf{Whittle Index}]
    Given indexablity and transition kernel $P_{n,t}$, the Whittle index $W_{n, t}(s; P_{n,t})$ of arm $n$ at state $s\in \mathcal S$ in episode $t$ is defined as:
    \begin{align}\label{Whittle}
       W_{n, t}(s; P_{n,t}):=\inf\{\lambda: Q_{n, \lambda, t}(s, 0)= Q_{n, \lambda, t}(s, 1)\}.
    \end{align}
\end{definition}
The Whittle index $W_{n, t}(s; P_{n,t})$ represents the minimum activation cost at which activating arm $n$ in state $s$ at episode $t$ is equally optimal to not activating it.


{\bf Whittle Index Policy} activates at most $M$ arms out of $N$ arms with the highest Whittle indices $W_{n, t}(s_{n, h, t}; P_{n,t})$. However, as we can observe from \eqref{Whittle}, we can compute the Whittle index if we know the transition kernel $P_{n,t}$ of every episode $t\in [T]$. In many applications, the transition kernel can be unknown and time-varying. For example, the transmission success probability of moving users can be unknown and time-varying, making the transition kernel in the wireless scheduling application unknown and time-varying. In Section \ref{wirelessscheduling}, we have discussed these phenomena in one wireless scheduling application.

\subsection{The Transition Kernel Model}

Now, we model how transition kernels change over every episode.

\subsubsection{Non-Stationarity:} In this section, we model the transition kernels for our non-stationary RMAB setting. We assume that the transition kernels $P_{n, t}$ may drift at varying rates across different arms $n \in [N]$ with the constraint that the total variation distance between transition kernels of two consecutive episodes is bounded from above by  
\begin{align}
    \max_{(s, a)\in \mathcal S \times A} \sum_{s'\in \mathcal S} \bigg|P_{n, t}(s'|s, a)- P_{n, t-1}(s'|s, a)\bigg|\leq \frac{V_n}{T},
\end{align}
where $V_n$ is the total variation budget across the entire $T$ episodes. The total variation budget $V_n$ represents the total non-stationarity in arm $n$ across the entire horizon, and is a standard quantity used for analzing dynamic regret in online learning literature \cite{ortner2020variational, cheung2020reinforcement}.

\subsubsection{Sparsity:} In many applications including wireless scheduling discussed in Section \ref{wirelessscheduling}, the probability transition kernels are sparse - meaning that many state transitions are not possible under certain actions. To model this we introduce $\mathcal S_{0}(s, a)$ as the set of all states $s' \in \mathcal S$ such that the probability to transit from state $s \in \mathcal S$ to state $s' \in \mathcal S$ given action $a \in \mathcal A$ is always $0$, i.e., $$\mathcal S_{0}(s, a)=\{s' \in \mathcal S: P_{n,t}(s^{\prime}|s, a)=0, \forall t\}.$$ The sets $\mathcal S_{0}(s, a)$ for all $(s, a) \in \mathcal S\times \mathcal A$ represents the sparsity of transition kernels for arm $n$. Our proposed algorithm can utilize this sparsity to reduce the complexity of the algorithm. Further, if we know sparsity (even approximately), we can use this information to learn faster by reducing exploration for certain transitions. Even in the absence of any sparsity, our results hold and the proposed algorithm are able to guarantee sublinear dynamic regret. {The DM is assumed to know the parameter $\mathcal S_{0}(s, a)$ for all $(s, a) \in \mathcal S\times \mathcal A$. 



\section{Online Learning of Whittle Indices with Known Variation Budget}\label{knownvariation}

To compute the Whittle index, we need to know transition kernels. In practice, transition kernels $P_{n,t}$ are unknown and non-stationary. In this section, we develop our {\it sliding window-based online Whittle policy} (SW-Whittle), provided in Algorithm \ref{alg:UCWWhittleknown} that (i) learns the probability transition kernels with known variation budget $V_n$ and (ii) uses them to compute the Whittle Index to pick approximately optimal policies in each episode.

Our Algorithm \ref{alg:UCWWhittleknown}, is motivated by the {\it UCWhittle} approach proposed in \cite{wang2023optimistic}. However, the {\it UCWhittle} policy is designed for static settings and does not handle time-varying transition kernels. This motivates the two main technical innovations in our policy. First, we employ a sliding window method that tracks the transition kernels of the past $w_n$ episodes, rather than all past episodes. The parameter $w_n$ is decided based on the total variation budget $V_n$. Second, we change the confidence bound provided in \cite{wang2023optimistic}. In designing the new confidence bound, we 
add a prediction horizon $w_n V_n/T$. We also discuss in Section \ref{unknownVariation} how we estimate the total variation budget $V_n$.

\subsection{Sliding Window and Confidence Bounds}
At each episode $t$ and for each arm $n$, we maintain variables $C_{t, w}^{(n)}(s', a, s)$, which count the number of transitions from state $s$ to the state $s'$ via the action $a$ observed within the past $w$ episodes, i.e. the sliding window. By using the counts for past $w$ episodes, we compute the empirical transition probabilities 
\begin{align}
    \hat P_{n, t, w}(s'|s,a):=\frac{C_{t, w}^{(n)}(s', a, s)}{C^{(n)}_{t, w}(s, a)},
\end{align}
 where we define $C^{(n)}_{t, w}(s, a):=\max\bigg\{\sum_{s^{\prime} \in \mathcal S} C_{t, w}^{(n)}(s', a, s), 1\bigg\}.$ {Using the upper confidence bound approach, we consider the following confidence radius
\begin{align}\label{confidencebound}
    d_{t}^{n}(s, a)=\sqrt{\frac{2|\mathcal S|\mathrm{log}(2|\mathcal S||\mathcal A|NT/\eta)}{C^{(n)}_{t,w}(s, a)}}+\frac{w_nV_n}{T},
\end{align}
where $0<\eta<1$ is a design parameter. The term $\frac{w_nV_n}{T}$ in the confidence radius $d_{t}^{n}(s, a)$ measures how far the transition kernels could have drifted over a window of $w_n$ episodes.}

Equipped with these definitions, the ball $B_{t}^{(n)}$ of the possible values for transition probabilities $P_{n,t}(s'|s, a)$ at any episode $t$ can be characterized as follows
\begin{align}\label{UCB}
    B_{t}^{(n)}=\bigg\{P_{n,t}: &\sum_{s' \in \mathcal S}\bigg|P_{n,t}(s'|s, a)-\hat P_{n,t, w_i}(s'|s, a)\bigg|\leq d_{t}^{(n)}(s, a),\nonumber\\ & P_{n, t}(s'|s, a)=0, \forall s' \in \mathcal S_0(s, a),\nonumber \\ &\sum_{s'\in \mathcal S} P_{n, t}(s'|s, a)=1, \forall (s, a) \in \mathcal S\times\mathcal A\bigg\}.
\end{align}
We will show later that the true transition kernel lies within this high-dimensional ball with high probability in each episode.

\begin{algorithm}[t]
\SetAlgoLined
\SetAlgoNoEnd
\SetKwInOut{Input}{input}
\caption{SW-Whittle Policy with known Budget $V_n$}\label{alg:UCWWhittleknown}
\Input{State Space $\mathcal S$, Action Space $\mathcal A$, Reward Function $r_n(s,a)$ for all $(s, a)$ and arms $n$} 
{DM initializes a Lagrange cost $\lambda_1$}\\
\For{every episode $t=1,2, \ldots, T$}
{DM decides window size $w_n=\lceil (T/V_n)^{2/3}\rceil$ for all $n \in [N]$\\
Arm $n$ starts with state $s_{n, 0}$\\
{DM predicts $\tilde P_{n,t}$ for all arm $n \in [N]$ using \eqref{optimistic1} with $\lambda_t$.}\\ 
{DM computes Whittle Index $W_{n, t}(s, \tilde P_{n, t}), \forall s\in \mathcal S, n \in [N]$ using \eqref{Whittle}.}\\
\For{$h=1, 2, \ldots, H$}
{DM activates $M$ arms (i.e., action=1) with highest Whittle Indices $W_{n, t}(s_{n,h, t}, \tilde P_{n, t})$.  \\
Observes the next state $s_{n, h+1, t} \sim P_{n,t}(\cdot|s_{n,h,t}, a_{n,h,t})$\\
}
{Update $\lambda_{t+1}\gets$ M-th highest Whittle Index}\\
DM updates counts $C_{t, w_n}^{(n)}$\\
}

\end{algorithm}

\subsection{Online Whittle Indices}
We predict the transition probabilities in an optimistic approach. We select the optimistic transition probability $\tilde P_{n, t}$ for each arm $n$ that maximizes the value function within the confidence bound. The optimization problem for predicting the transition probability $\tilde P_{n, t}$ is given by

\begin{align}\label{optimistic1}
    &\max_{P_{n, t} \in B^{(n)}_{t}} V_{n, \lambda, t}(s), \nonumber\\
    &\quad \mathrm{s.t.}~V_{n, \lambda, t}(s)=\max_{a \in \mathcal A} Q_{n, \lambda, t}(s, a), \nonumber\\
    &\quad\quad \quad Q_{n, \lambda, t}(s, a)=r_n(s,a)-\lambda a + \sum_{s'}P_{n, t}(s'|s,a) V_{n, \lambda, t}(s').
\end{align}

To solve \eqref{optimistic1}, we use the extended value iteration described in \cite{auer2008near}. In every iteration of the extended value iteration method, for every state-action pair $(s, a) \in \mathcal S \times \mathcal A$, we require to solve an inner optimization problem
\begin{align}
    \max_{P_{n, t} \in B^{(n)}_{t}} \sum_{s'}P_{n, t}(s'|s,a) V_{n, \lambda, t}(s').
\end{align}

The extended value iteration algorithm first sorts the set of states $\mathcal S= \{s_1,s_2, \ldots, s_{|\mathcal S|}\}$ in descending order of their value-functions, such that $V_{n,t}(s_i) \geq V_{n,t}(s_{i+1})$. Next, we initialize the transition probabilities by setting 
$$\tilde P_{n, t}(s_1|s,a) =\hat P_{n, t}(s_1|s,a) + \frac{d_t^{(n)}(s,a)}{2}$$ for the state with the highest value function, and $$\tilde P_{n, t}(s_i|s, a)=\hat P_{n, t}(s_i|s,a)$$ for all other states $i > 1$. If the resulting probabilities sum to more than 1, we restore validity by iteratively reducing
the probability of the state with the lowest value function. We start with $l= |\mathcal S|$ and set 
$$\tilde P_{n, t}(s_l|s, a)=\max\bigg\{0,1-
\sum_{j \neq l}P_{n, t}(s_j|s,a)\bigg\}.$$ This process is repeated for decreasing
$l$ until the transition probabilities sum to 1.

As a result of the maximization procedure of \eqref{optimistic1}, the true value function is upper bounded by the value function under the predicted transition kernel provided that the confidence bound in \eqref{UCB} holds. This upper bound value function will later allow us to prove regret bounds. Using the predicted transition kernel $\tilde P_{n, t}$, we compute $W_{n, t}(s: \tilde P_{n, t})$, the Whittle index of state $s\in \mathcal S$ for arm $n$ as defined in \eqref{Whittle}. Finally, we update Lagrange multiplier $\lambda^{t+1}$ as the $M$-th highest Whittle index at time slot $H$ of episode $t$. 

\subsection{Regret Analysis}

It is well known in RMAB literature, the problem of obtaining an optimal policy for an RMAB
problem is generally intractable \cite{papadimitriou1994complexity}. Consequently, using an optimal policy for an RMAB is an
impractical benchmark to evaluate the performance of online learning. To address this issue, \cite{tripathi2021online} and \cite{wang2023optimistic} used the Lagrangian relaxed problem or alternatively the best Whittle index policy to assess the
performance of an online learning algorithm. Following \cite{wang2023optimistic}, we will also employ the Lagrangian relaxed problem to evaluate the online learning performance in this paper.

The cumulative dynamic regret in $T$ episodes is given by 
\begin{align}\label{Regret1}
\mathrm{Reg}(T)=\sum_{t=1}^T\mathcal L(\pi^*_t, (P_{n, t})_{n=1}^N, \lambda^*_t)-\mathcal L(\pi_t, (P_{n, t})_{n=1}^N, \lambda_t),
\end{align}
where the Lagrangian $\mathcal L(\cdots)$ is defined in \eqref{Lagrangian} and $(\pi^*_t, \lambda^*_t)$ are the optimal solution of 
\begin{align}
    \min_{\lambda\geq 0}\max_{\pi_t\in \Pi} 
     \mathcal L(\pi_t, (P_{n, t})_{n=1}^N, \lambda).
\end{align}

To create a regret bound, we first need to establish how good our estimates of the time-varying transition kernel are. To do so, we will bound the probability that the true kernel is outside the high-dimensional ball $B_{t}^{(n)}$ introduced in \eqref{UCB}. Lemma 1 describes the result in detail. 

\begin{lemma}\label{lemma1}
    Given $0<\eta<1$ in \eqref{confidencebound}, the probability that the true kernel $P_{n,t}$ lies within the high-dimensional Ball $B_{t}^{(n)}$ (described by eq. \ref{UCB}) is greater than or equal to $1-\eta$, i.e., $\mathrm{Pr}( P_{n,t} \in B_{t}^{(n)}, \forall n, \forall t) \geq 1-\eta$.
\end{lemma}
\begin{proof}
See Appendix \ref{plemma1}
\end{proof}

Lemma 1 implies that for \textit{every} episode $t$, we can provide a confidence region in which the true transition kernel will lie with high probability.

Next, using this result, Theorem~\ref{theorem1} characterizes the upper bound for $\mathrm{Reg}(T)$. 

\begin{theorem}\label{theorem1}
    Given $0<\eta<1$ in \eqref{confidencebound} and any window size $w_n$, the cumulative dynamic regret of Algorithm \ref{alg:UCWWhittleknown}, with probability $1-\eta$, satisfies:
    \begin{align}
        \mathrm{Reg}(T)\leq& \sum_{t=1}^T O\bigg(\sum_{n=1}^N 2|\mathcal S| G_{t,n}(w_n)+w_nV_nH/T\bigg),
    \end{align}
where $G_{t,n}(w)=\max_{(s,a)\in \mathcal S \times \mathcal A}g_{t,n}(s, a, w)$, the function $g_{t,n}(s, a, w)= \mathbb E_{P_{n, t}, \pi_{n, t}}[ \alpha_{t}^{(n)}(s, a)/\sqrt{C_{t, w}^{(n)}(s,a)}]$ is non-increasing in $w$, and $\alpha_{t}^{(n)}(s, a)$ is a random variable that denotes the number of visit at $(s,a)$ in episode $t$. 
\end{theorem}
\begin{proof}[Proof Sketch] We focus on bounding the dynamic regret when the confidence bound  holds.
\begin{align}
    \mathrm{Reg}^{(t)}
    &=\sum_{t=1}^T\mathcal L(\pi^*_t, (P_{n, t})_{n=1}^N, \lambda^*_t)-\mathcal L(\pi_t, (P_{n, t})_{n=1}^N, \lambda_t)\nonumber\\
    &\leq \sum_{t=1}^T\mathcal L(\pi^*_t, (P_{n, t})_{n=1}^N, \lambda_t)-\mathcal L(\pi_t, (P_{n, t})_{n=1}^N, \lambda_t)\nonumber\\
    &=\sum_{t=1}^T\sum_{n=1}^NU(\pi_{n, t}^*, P_{n,t}, \lambda_t)-U(\pi_{n, t}, P_{n,t}, \lambda_t)\nonumber\\
    &\leq \sum_{t=1}^T\sum_{n=1}^NU(\pi_{n, t}, \tilde P_{n,t}, \lambda_t)-U(\pi_{n, t}, P_{n,t}, \lambda_t), 
 \end{align}    
where the first inequality holds because $\lambda^*_t$ minimizes $\mathcal L(\pi^*_t, (P_{n, t})_{n=1}^N, \lambda)$ for all $
\lambda\geq 0$ and the second inequality holds because of \eqref{optimistic1}. Then, by applying value difference theorem \cite[Theorem 6.4] {wang2023optimistic}, we have 

\begin{align}
    &U(\pi_{n, t}, \tilde P_{n,t}, \lambda_t)-U(\pi_{n, t}, P_{n,t}, \lambda_t)\nonumber\\
    &\!\leq \mathbb E_{P_{n, t}, \pi_{n, t}} \bigg[ \sum_{(s,a)\in \mathcal S} \alpha_{t}^{(n)}(s, a) \sum_{s'\in \mathcal S}\bigg|\tilde P_{n, t}(s'|s, a)-P_{n, t}(s'|s, a)\bigg|\bigg]V_{max}\nonumber\\
    &=\mathbb E_{P_{n, t}, \pi_{n, t}} \bigg[ \sum_{(s,a)\in \mathcal S} \alpha_{t}^{(n)}(s, a) \sum_{s'\in \mathcal S}d_{t}^{(n)}(s,a)\bigg]V_{max}
\end{align}
where $V_{max}=\max_{n \in [N], s \in \mathcal S} V_n(s'; \pi_{n,t}, P_{n,t})$. After this, by substituting the value of $d_{t}^{(n)}(s,a)$, we can obtain Theorem \ref{theorem1}. For a detailed proof, please see Appendix \ref{Ptheorem1}.
\end{proof}

At a high level, the regret bound in Theorem \ref{theorem1} is characterized by a fundamental trade-off. For a fixed state space $\mathcal {S} $, the first term depends on $G_{t,n}(w_n)$, a non-increasing function of the window size $w_n$. Conversely, the second term scales linearly with $w_n$. This structure highlights the necessity of carefully tuning the window size $w_n$. A window that is too small leads to a high learning error $G_{t,n}(w)$ because of an insufficient number of samples $C_{t, w_n}^{(n)}(s,a)$ for every feasible state-action pair $(s,a)$. However, if we increase the window size $w_n$,  the algorithm fails to track non-stationary changes. 

In Theorem \ref{theorem2}, we set the window size $w_n$ as a function of the total variation distance $V_n$ and find the dynamic regret with respect to number of episodes $T$ and the total variation budget $V_n$. 

\begin{theorem}\label{theorem2}
    Given $0<\eta<1$ in \eqref{confidencebound}, if there exists a positive probability to visit every $(s,a) \in \mathcal S$ at least once in every episode $t\in [T]$ for all arms $n\in [N]$, then, with probability $1-\eta$, the cumulative dynamic regret of Algorithm \ref{alg:UCWWhittleunknown} satisfies:
    \begin{align}
        \mathrm{Reg}(T)\leq \sum_{n=1}^N \tilde O(T^{2/3}V_n^{1/3}).
    \end{align} 
\end{theorem}
\begin{proof}
See Appendix \ref{Ptheorem2}.
\end{proof}
Theorem \ref{theorem2} establishes a sub-linear dynamic regret with the number of episodes $T$ when the window size is $w_n=\lceil (T/V_n)^{2/3}\rceil$. For a detailed proof of Theorem \ref{theorem2}, please see Appendix A.3 of our technical report \cite{shisher2025Technical}.

However, to determine the window size $w_n$, it is necessary to know the total variation distance $V_n$, which may not be known beforehand. In the next section, we will discuss how to estimate the total variation distance $V_n$ and the impact of estimation error on dynamic regret. 

%% file: result.tex
\section{Online Learning of Whittle Indices with unknown Variation Budget}\label{unknownVariation}

{We adopt the Bandit over Bandit approach \cite{cheung2022hedging, wei2023provably} for the estimating the variation budget $V_n$. 

\begin{algorithm}[t]
\SetAlgoLined
\SetAlgoNoEnd
\SetKwInOut{Input}{input}
\caption{SW-Whittle Policy with unknown Budget $V_n$}\label{alg:UCWWhittleunknown}
\Input{State Space $\mathcal S$, Action Space $\mathcal A$, Reward Function $r_n(s,a)$ for all $(s, a)$ and $J_n$ for all arms $n$} 
{Initialize a Lagrange cost $\lambda_{1}$ and a set of $J_n$ budget values, where $i$-th budget is $V_n(i)=V_{max}-(i-1)V_{max}/J_n$}\\
{Initialize parameters $\beta_n\in [0, 1]$, $\alpha_{n, i}=1$ and $X_{n, i}=0$}\\
\For{every episode $t=1,2, \ldots, T$}
{Set $p_{i, n}(t)=(1-\beta_n)\frac{\alpha_{n, i}}{\sum_{i=1}^{J_n} \alpha_{n, i}}+\beta_n/J_n$\\
Select $i_{n, t}\in\{1, 2, \ldots, J_n\}$ randomly according to probabilities $p_{1,n}(t), p_{2, n}(t), \ldots, p_{J_n, n}(t)$\\
Select window size $w_n=\lceil (T/V_n(i_{n, t})^{2/3}\rceil$\\
Arm $n$ starts with state $s_{n, 0}$\\
{Predict $\tilde P_{n,t}$ for all arm $n \in [N]$ using \eqref{optimistic1} with $\lambda_t$.}\\ 
{Compute Whittle Index $W_{n, t}(s; \tilde P_{n, t}) \forall s\in \mathcal S, n \in [N]$ with $\tilde P_{n,t}$ using \eqref{Whittle}.}\\
\For{$h=1, 2, \ldots, H$}
{DM activates $M$ arms (i.e., action=1) with highest Whittle Indices $W_{n, t}(s_{n,h, t}; \tilde P_{n, t})$.  \\
Observes the next state $s_{n, h+1, t} \sim P_{n,t}(\cdot|s_{n,h,t}, a_{n,h,t})$\\
}
Observe normalized episodic reward $\tilde r_{n}(t)$\\
Update $X_{n, i_{n, t}}\gets \tilde r_n(t)/p_{i_{n, t}, n}(t)$\\
Update $\alpha_{n, i_{n, t}}\gets \alpha_{n, i_{n, t}}\exp(\beta_n X_{n, i_{n, t}}/J_n)$\\
{Update $\lambda_{t+1}\gets$ M-th highest Whittle Index}\\
Updates counts $C_{t, w_n}^{(n)}$\\
}

\end{algorithm}

\subsection{Bandit over Bandit}
In this estimation approach, we solve another bandit problem to select $V_n$
from a finite set of possible budget values based on the history by using EXP3 algorithm \cite{auer2002nonstochastic}.  Our SW-Whittle algorithm with unknown variation budget $V_n$ is provided in Algorithm \ref{alg:UCWWhittleunknown}, where we incorporate EXP3 algorithm and our choice of finite set of possible budget values. 

We first create a set of possible budget values. Towards this objective, we get the maximum value for the variation budget as $V_{n, max}=2T$. The total variation budget $V_n$ is bounded by $2T$ in the worst-case scenario, as  
\begin{align}
    \max_{(s, a)\in \mathcal S \times A} \sum_{s'\in \mathcal S} \bigg|P_{n, t}(s'|s, a)- P_{n, t-1}(s'|s, a)\bigg| \leq 2,
\end{align}
where the equality holds in the extreme case when the consecutive distributions are completely disjoint. However, for sublinear regret to be achievable, we typically consider the regime where $V_n=o(T)$.

By using $V_{n, max}$, we can now define the set of quantized drift values as
$\{V_{n, max}, V_{n, max}-V_{n, max}/J_n, V_{n, max}-2V_{n, max}/J_n, \ldots, V_{n, max}/J_n\},$
where $J_n$ is the number of quantization levels. This approach of quantifying and approximately estimating drift values is novel within the online learning literature. We will show in Theorem \ref{theorem4} that the number of levels $J_n$ affects the dynamic regret (more levels means more accurate tracking of $V_n$ but also slower learning in the Bandit-over-Bandit approach).} 

\subsection{Regret Analysis}

Next, using this result, Theorem~\ref{theorem4} characterizes the upper bound for $\mathrm{Reg(T)}$.

Let $V_n$ be the actual total variation measure, given by 
\begin{align}
       V_n= \sum_{t=1}^T\max_{(s, a)\in \mathcal S \times A} \sum_{s'\in \mathcal S} \bigg|P_{n, t}(s'|s, a)- P_{n, t-1}(s'|s, a)\bigg|.
    \end{align}

\begin{theorem}\label{theorem4}
    Given $0<\eta<1$ in \eqref{confidencebound}, if there exists a positive probability to visit every $(s,a) \in \mathcal S\times \mathcal A$ at least once in every episode $t\in [T]$ for all arms $n\in [N]$, then, with probability $1-\eta$, the cumulative dynamic regret of Algorithm \ref{alg:UCWWhittleunknown} satisfies:
    \begin{align}
        \mathrm{Reg(T)}\leq \sum_{n=1}^N\tilde O(T^{2/3}( V_n+2T/J_n)^{1/3})+\tilde O(\sqrt{TJ_n}).
    \end{align} 
\end{theorem}
\begin{proof}
See Appendix \ref{Ptheorem4}.
\end{proof}

The regret bound in Theorem \ref{theorem4} involves two error components: The first term is the learning error of RMAB due to unknown transition kernels. If the actual total variation distance $V_n$ increases, the learning error increases. For a fixed total variation distance, we achieve sub-linear dynamic regret with the number of episodes $T$; The second term is a variation budget estimation error (BoB learning error) that increases with $J_n$. Overall, our algorithm can achieve sub-linear regret ($T^{2/3}$) with the number of episodes $T$ in the regime where $V_n=o(T)$.
Our dynamic regret $\tilde{O}(T^{2/3}V_n^{1/3})$ and the considered regime $V_n=o(T)$ align with established results for online learning in non-stationary MDPs  \cite{ortner2020variational, cheung2020reinforcement, marin2024metacurl, wei2023provably, wei2021non}.

%% file: application.tex
\section{Wireless Scheduling Application}\label{wirelessscheduling}
\begin{figure*}[t]
    \centering
    \begin{subfigure}[b]{0.4\textwidth}
        \centering
        \includegraphics[width=\textwidth]{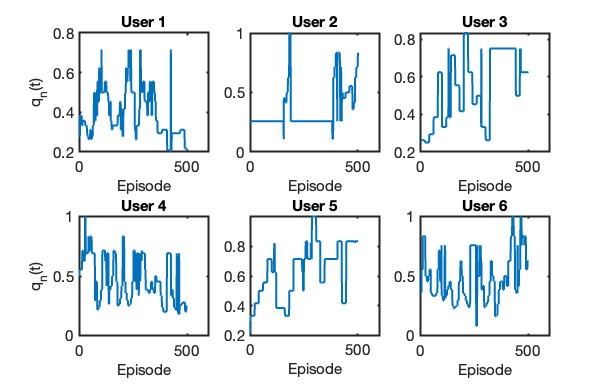}
        \caption{$q_n(t)$ vs. $t$ }
        \label{fig:successprob}
    \end{subfigure}
%
 \hspace{1mm} 
    \begin{subfigure}[b]{0.40\textwidth}
        \centering
        \includegraphics[width=\textwidth]{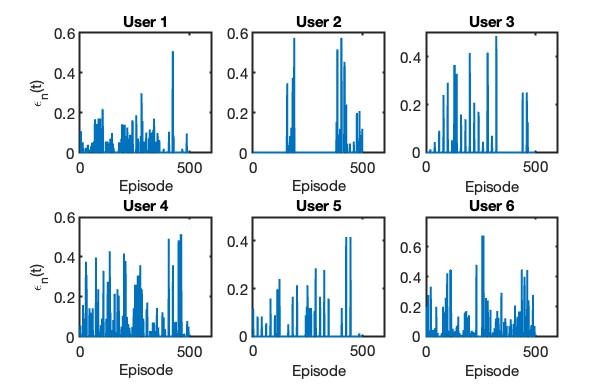}
        \caption{$\epsilon_n(t)=|q_n(t)-q_n(t-1)|$ vs $t$}
        \label{fig:variation}
    \end{subfigure}
    
    \caption{Analysis of channel state variation over 500 episodes: (a) Transmission success probability $q_n(t)$ vs. Episode $t$ and (b) Variation of transmission success probability $\epsilon_n(t)=|q_n(t)-q_n(t-1)|$ vs Episode $t$.}
    \label{fig:combined_metrics}
\end{figure*}

\begin{table*}[t]
    \centering
    \begin{tabular}{l c c c c c c}
        \toprule
        \textbf{Applications}  & \textbf{$(N, M)$} & \textbf{Our Policy} & \textbf{UCWhittle}  & \textbf{UCWhittle+Win} & \textbf{Random} & \textbf{WIQL} \\
        \midrule
        {1-D Bandit} & (10, 1) & {\bf 957}{\tiny $\pm 155$} &  6528{\tiny $\pm 1996$}&  6377{\tiny $\pm 1452$} & 11916{\tiny $\pm  2154$} & 12060{\tiny $\pm 2226$} \\
        & (10,4) & {\bf 2119}{\tiny  $\pm 237$} &  27620 {\tiny $\pm 5850$} & 21628 {\tiny $\pm 3054$} &   28349 {\tiny $\pm 4668$}  &  28068{\tiny $\pm 4601$} \\
        & (20,4) & {\bf 2065}{\tiny  $\pm  368$} & 28985 {\tiny $\pm 7417$} &  24405 {\tiny $\pm 7286$} &   39358 {\tiny $\pm  6145$}  & 40314-{\tiny $\pm 6491$} \\
        \midrule
        {Scheduling} &  (10,1) & {\bf  503}{\tiny $\pm 31$} &  981{\tiny $\pm 55$} & 787{\tiny $\pm 42$} &  3239{\tiny $\pm 115$} &  3408 {\tiny $\pm 98$}\\
        {(Synthetic)}& (10, 4) & {\bf 945}{\tiny $\pm  94$} &  1183{\tiny $\pm 152$} &  1095{\tiny $\pm 118$} &   2598{\tiny $\pm  49$} &  2216{\tiny $\pm 64$}\\
        & (20,4) & {\bf 1276}{\tiny $\pm  90$} & 2097{\tiny $\pm 189$} & 1808{\tiny $\pm   138$} & 7094{\tiny $\pm 189$} &  6397{\tiny $\pm 279$}\\
         \midrule
        {Scheduling (Real)} &  (6,1) & {\bf  2003}{\tiny $\pm 32$} &  4821{\tiny $\pm 176$} & 4772{\tiny $\pm 99$} &  10539{\tiny $\pm 80$} &  5171 {\tiny $\pm 163$}\\
        & (6,3) & {\bf  1964}{\tiny $\pm 35$} &  3544{\tiny $\pm 73$} & 3504{\tiny $\pm 73$} &  19587{\tiny $\pm 38$} &  4163 {\tiny $\pm  72$}\\
        \bottomrule
    \end{tabular}
    \caption{$\mathrm{Reg}(T)$ for different values of $N$ and $M$.}\label{tab:performance}
    \vspace{-3mm}
\end{table*}

In this section, we apply the results developed in the previous sections to solve an online wireless resource allocation problem for remote estimation systems. 

Consider a remote estimation system, where a receiver estimates $N$ time-varying targets observed by $N$ sources. To estimate the target of source $n$, the receiver uses the most recent observation delivered from the source $n$. Due to limited communication resources (e.g., orthogonal channel frequency), only $M$ of the $N$ sources can transmit their current observation in any given time-slot. Moreover, the communication channel between a source and the receiver can be unreliable and the channel states can vary over time. We consider independent Bernoulli channels between every source and the receiver, with probability of success $q_n(t)$ for source $n$ at each time slot $h$ of episode $t$. 

Let $a_{n, h, t}\in \{0, 1\}$ be scheduling decision of source $n$ at time slot $h$ of episode $t$. if $a_{n, h, t}=0$, the source $n$ is not allowed for transmission; otherwise, If $a_{n, h, t}=1$, the source $n$ is scheduled for transmission at time $t$ and the transmitted observation is delivered in the next time slot $t+1$. 

For every source $n$, the age of information or simply AoI at the receiver, denoted by $\Delta_{n, h, t}$, measures the time difference between the current time $h$ and the generation time of the most recently delivered observation from source $n$ at episode $t$. We assume active sources, i.e. in any time-slot, any source $n$, when it is scheduled for transmission, generates a fresh update at will and sends. Let $u_{n, h, t}$ be a Bernoulli random variable with parameter $q_n(t)$ that denotes channel reliability between the $n$-th source and the receiver. Then, we have 
\begin{align}\label{AoIProcess}
\!\! \Delta_{n, h+1, t}=
  \begin{cases}
        \Delta_{n, h, t}+1, &\text{if}~a_{n, h, t}\neq 1~\text{or}~ u_{n, h, t}=0, \\
        1, &\text{if}~a_{n, h, t}= 1~\text{and}~ u_{n, h, t}=1.
    \end{cases}
    \end{align} 

Prior works established that estimation accuracy can be represented by function of AoI values. We consider general reward functions of AoI as
our metric of interest. For each source $n$, let $r_n(\Delta_{n, h, t})$ denote a reward function for estimation accuracy. Let $\pi$ be a scheduling policy that decides which sources to schedule in every time-slot. The AoI process $\Delta_{n, h, t}$ depends
on $\pi$ and the channel processes. Then, the corresponding expected discounted sum of rewards in episode $t$ is given by 
\begin{align}\label{objectivewireless}
    R_t\left(\pi_t, (P_{n,t})_{n=1}^N\right):=\mathbb E\left[\sum_{h=1}^H \sum_{n=1}^N  \gamma^{h-1} r_n(\Delta_{n, h, t})\bigg| \pi_{t}, (P_{n,t})_{n=1}^N\right],
\end{align}
where the next-state transition $P_{n, t}(\delta'|\delta, a)$ is governed by the transition dynamics \eqref{AoIProcess} and depends on the parameter $q_n(t)$ and the scheduling decision $a_{n, h, t}=a$: when the source is scheduled, i.e., $a=1$, AoI drops to $1$ with probability $P_{n, t}(\delta'=1|\delta, 1)=q_n(t)$ and increases to $\delta+1$ with probability $P_{n, t}(\delta'=\delta+1|\delta, 1)=1-q_n(t)$; otherwise, when the source is not scheduled, i.e., $a=0$, AoI always increases by $1$, i.e., $P_{n, t}(\delta'=\delta+1|\delta, 1)=1$. 

From the above discussion, we can see that the probability transition kernels for the wireless scheduling problem are sparse. We have 

$$\mathcal S_0(\delta, 1)=\{\delta\in \mathbb N: \delta \neq 1,\delta+1\}$$
and 
$$\mathcal S_0(\delta, 0)=\{\delta\in \mathbb N: \delta \neq \delta+1\}.$$
Moreover, we only need to determine the transmission success probability $q_n(t)$ to get the transition kernels $P_{n, t}$. Hence, we can have empirical success probability $q_n(t)$ as 
\begin{align}
    \hat q_n(t):=\frac{\sum_{\delta \in \mathbb N} C_{t, w}^{(n)}(1, a, \delta)}{\sum_{\delta\in \mathbb N} C^{(n)}_{t, w}(\delta, a)}
\end{align}
and substitute this value for any $\hat P_{n, t}(1|\delta,1)$. Moreover, if $r_n(\delta)$ is a non-increasing function of AoI $\delta$, we get the closed form solution of \eqref{optimistic1} as 
\begin{align}
    \tilde P_{n, t}(1|\delta,1)=\min\bigg\{\hat q_n(t)+\frac{d_t^{(n)}(\delta, 1)}{2}, 1\bigg\}
\end{align}
and 
\begin{align}
    \tilde P_{n, t}(\delta+1|\delta,1)=1-P_{n, t}(1|\delta,1).
\end{align}

Then, using the predicted transition kernel, we get the Whittle index \eqref{Whittle}. Whittle index for the AoI function optimization problem with the predicted transition kernel can be obtained in closed-form or with very low-complexity \cite{Tripathi2019, shishertimely}.


%% file: simulations.tex
\section{Simulation Results}

\begin{figure*}[t]
  \centering
  
  \begin{subfigure}[t]{0.30\textwidth}
\includegraphics[width=\textwidth]{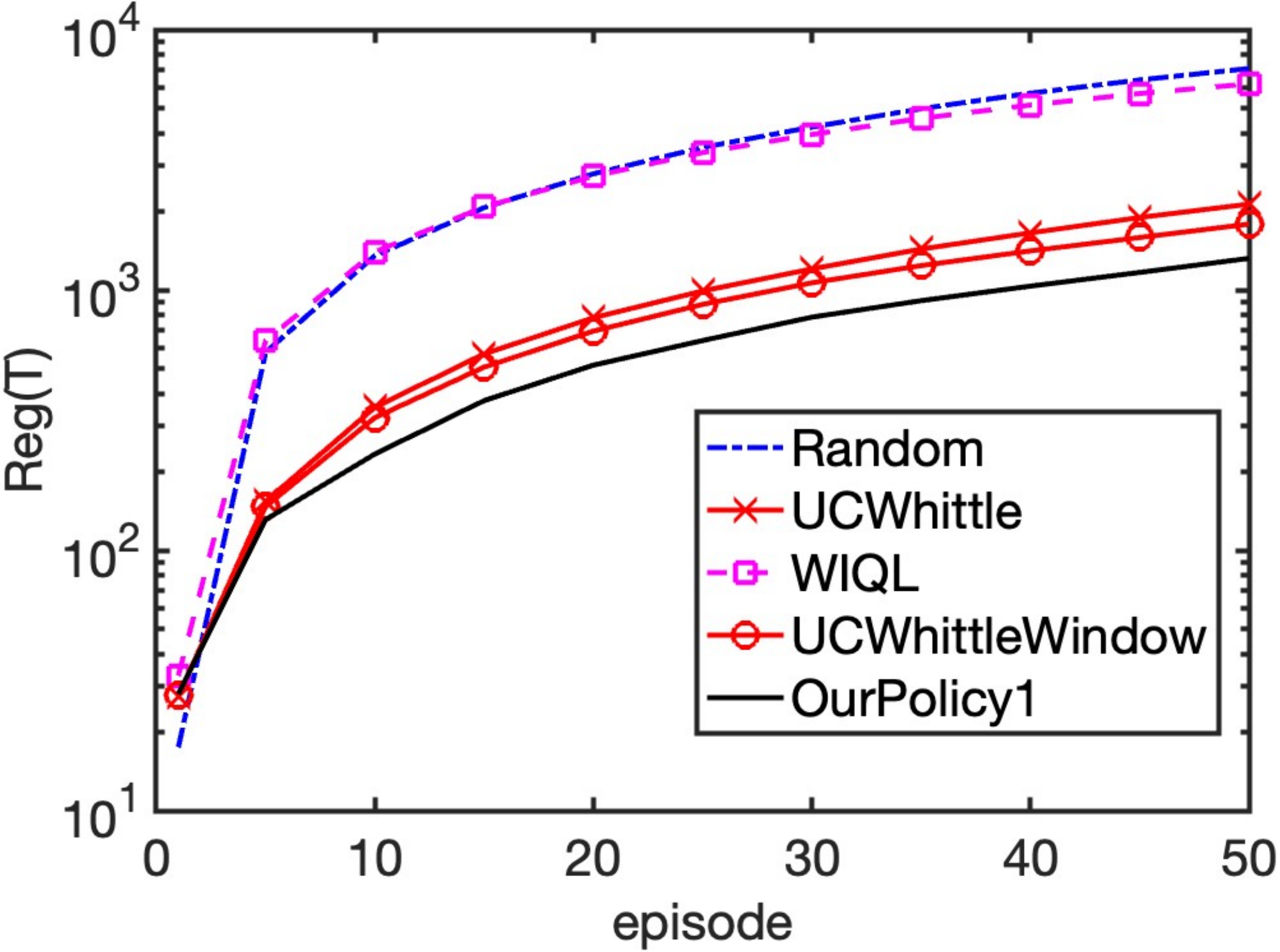}
  \subcaption{Scheduling (Synthetic).}
\end{subfigure}
%
\hspace{1mm} 
  \begin{subfigure}[t]{0.30\textwidth}
\includegraphics[width=\textwidth]{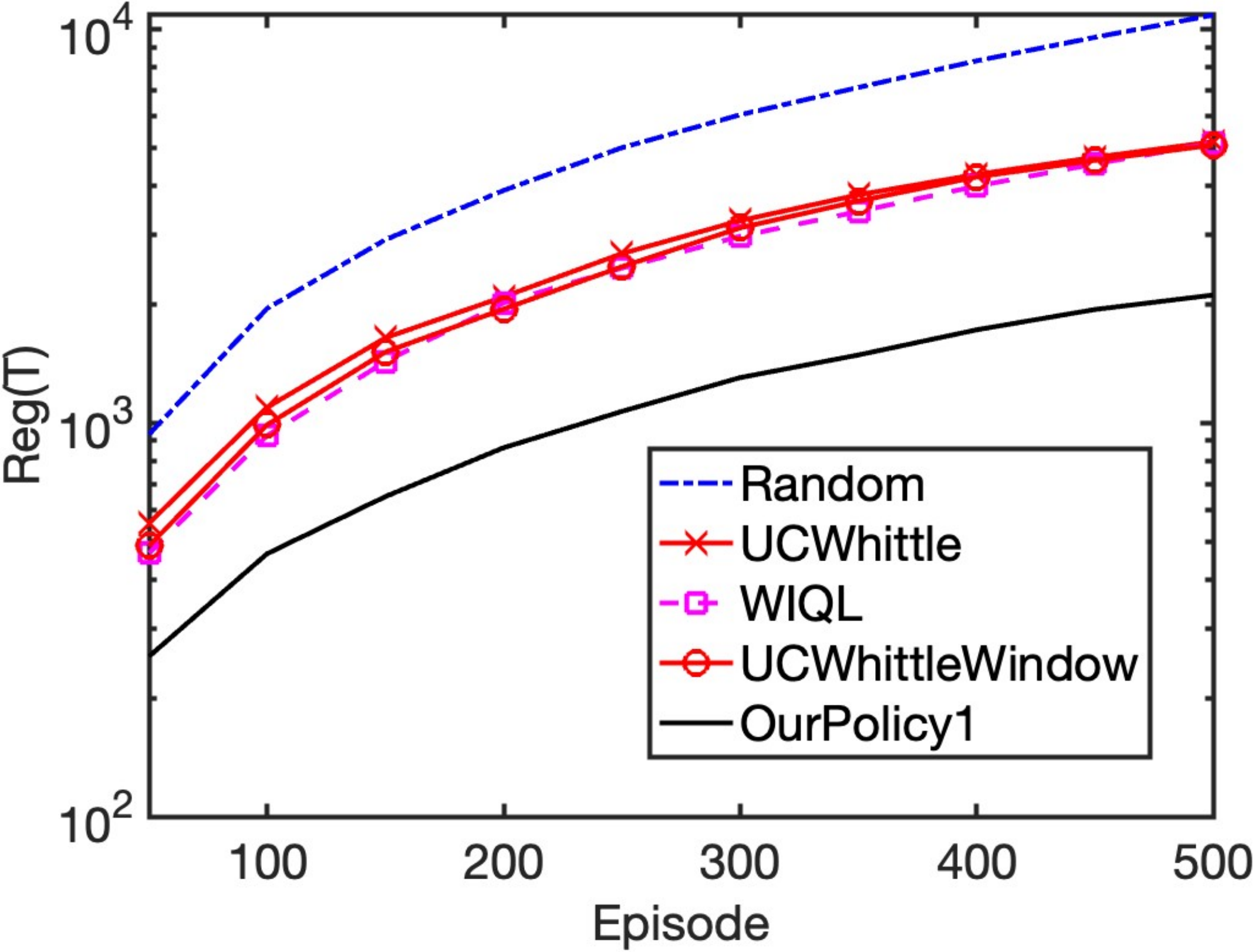}
  \subcaption{Scheduling. (Real Data).}
\end{subfigure}
%
\hspace{1mm}
\begin{subfigure}[t]{0.30\textwidth}
\includegraphics[width=\textwidth]{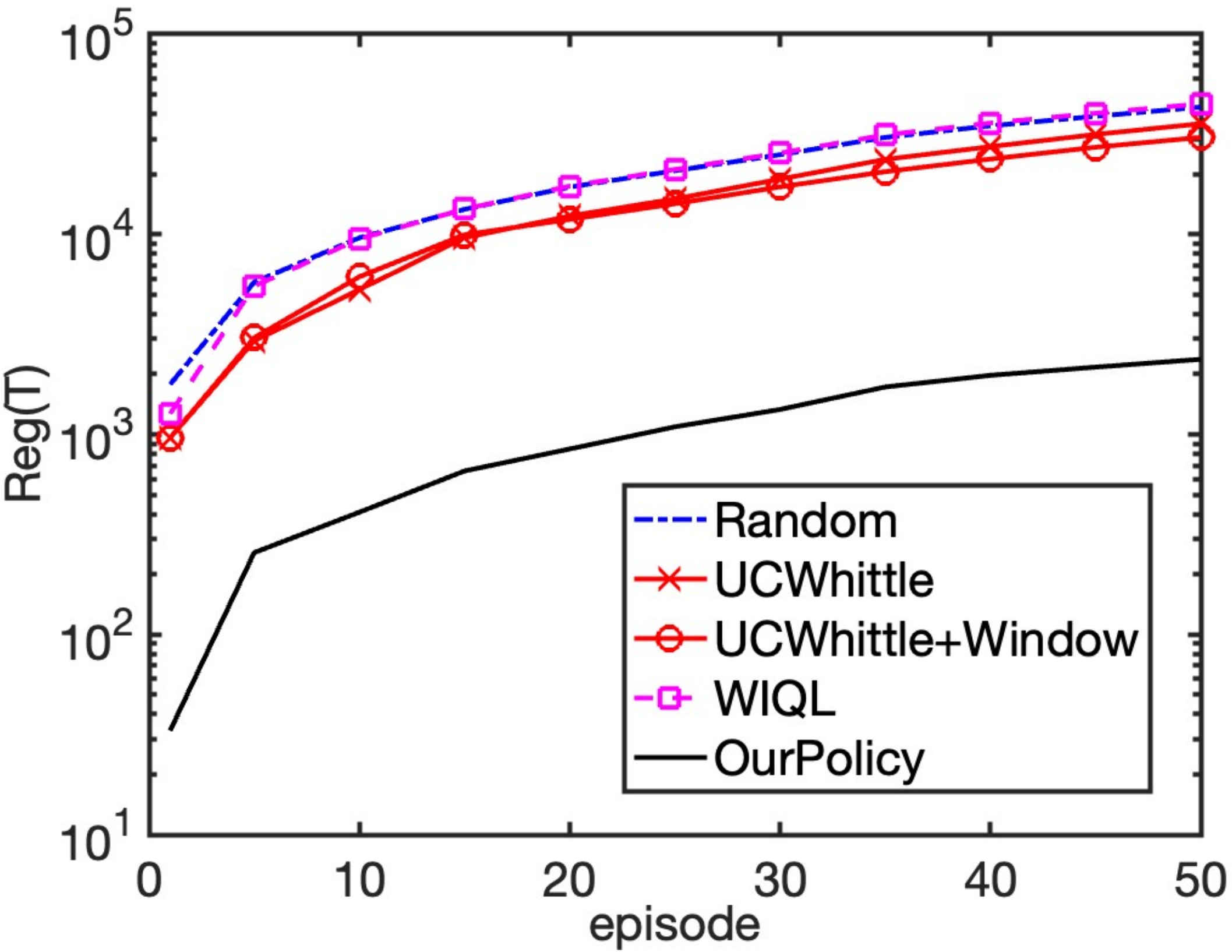}
\subcaption{\small 1-D Bandit.}
\end{subfigure}
\vspace{-0.05in}
\caption{\small $\mathrm{Reg(T)}$ Vs. number of episodes in Scheduling and 1-D Bandit.\label{fig:result}}
\vspace{-0.1in}
\end{figure*}

In this section, we demonstrate the performance of our proposed policy by evaluating it under two applications (wireless scheduling and one-dimensional bandit) modeled as RMAB. 

\subsection{Wireless Scheduling Model}
For wireless scheduling, we follow the system model discussed in Section \ref{wirelessscheduling}. Our goal is to optimize the objective defined in \eqref{objectivewireless}, where we model the reward for accurate timely estimation as the mutual information between the estimated signal and the actual signal. \cite{SunNonlinear2019} showed that mutual information can be determined using a decreasing function $- (a\mathrm{log}_2(1-\sigma_n^{2\Delta_{n, h, t}}))/2$ of the Age of Information (AoI) $\Delta_{n, h, t}$ for zero-mean i.i.d. Gaussian random variables with variance $\sigma^2_n$, where AoI $s_{n, h, t}$ of the source $n$ is the time difference between the current time $h$ and the generation time of the most recently delivered signal. The transmission success probability $q_n(t)$ of each user $n$ is generated both synthetically and using real world dataset.  

\subsubsection{Wireless Scheduling with Synthetic Success Probability}
In our experiment with synthetic success probability, we assume that $q_n(t)$ is unknown and non-stationary for half of the sources, whereas it is is unknown but stationary for the remaining half. For non-stationary arms, the probability of successful transmission $q_n(t)$ changes to $\min\{q_n(t-1)+\frac{V_n}{2T}, 1\}$ with probability $0.6$, or it changes to $\max\{q_n(t-1)-\frac{V_n}{2T}, 0\}$ with probability 0.4; the initial value of $q_n(t)=0.1$ is used. For the other half, $q_n(t)=1$ is unknown but stationary. In this experiment, we consider $a=1$ for reward function, the variance of $\sigma^2_n=0.9$ and $\sigma^2_n=0.5$ are used, respectively for non-stationary and stationary users. 

\subsubsection{Wireless Scheduling with Real-world Success Probability}
To get real-world success probability, we have also incorporated a recent dataset from \cite{reddy2025wioc} for the wireless scheduling problem. The dataset contains traces of measured signal strength for six users across indoor and outdoor settings, leading to non-stationary behavior. The signal strength time-series values allow us to calculate packet transmission success probabilities, which we then utilize to setup our wireless scheduling problem. In Figure \ref{fig:successprob}, we plot these success probability values, which clearly demonstrate the highly time-varying nature of the dataset due to user mobility. We also plot variation $\epsilon_n(t)=|q_n(t)-q_n(t-1)|$ in success probability in Figure \ref{fig:variation}. In this simulation, we consider six users with AoI function $-a\mathrm{log2}(1-0.9^{\Delta_{n,h, t}})$, where we set $a=0.4$ for three users, $a=0.5$ for one user, $a=0.9$ for the other two users. 

\subsection{One-dimensional Monotonic Bandit}
Now, we discuss how we model the one-dimensional monotonic bandit application. We consider a modified version of the one dimensional RMAB problem studied in \cite{killian2021q, nakhleh2022deeptop}. Each arm $n$ is a Markov process with $K$ states, numbered as $0, 1, \ldots, K-1$. For our simulations, we set $K=10$. The reward of an arm increases linearly with the current state, i.e. $r(s, a)=s$. If the arm is activated, then it can evolve from state $s$ to $\min\{s+1, K-1\}$ with probability $q_n(t)$ or remain in the same state $s$ with probability $1-q_{n}(t)$. If the arm is not activated, then it evolves from state $s$ to $\max\{s-1, 0\}$ with probability $p_n(t)$ or remain in the same state $s$ with probability $1-p_n(t)$. One-dimensional MDPs of this form are often used in health monitoring and machine monitoring applications \cite{matsena2021review, parisi2024monitored}. In our simulation, we consider - (i) $V_n=35$ over $T=50$ episodes, (ii) $p_n(t)$ changes to $\min\{p_n(t-1)+\frac{V_n}{4T}, 1\}$ with probability $0.5$, or it changes to $\max\{p_n(t-1)-\frac{V_n}{4T}, 0\}$ with probability 0.5 and (iii) $q_n(t)$ changes to $\min\{q_n(t-1)+\frac{V_n}{4T}, 1\}$ with probability $0.5$, or it changes to $\max\{q_n(t-1)-\frac{V_n}{4T}, 0\}$ with probability $0.5$.

\subsection{Performance Evaluation}
In each application, we consider that there are $N$ arms and a policy can activate $M$ of them in each time slot $h\in[H]$ of every episode $t\in [T]$. We evaluate our policy against the UCWhittle policy \cite{wang2023optimistic}, UCWhittle + Window policy, where we incorporate sliding window to UCWhittle and the window size is taken randomly, the WIQL policy \cite{biswas2021learn}, and a randomized policy \cite{Kadota2018}. The results 
are averaged over $50$ independent runs. In simulation results of Table \ref{tab:performance} and Figure \ref{fig:result}, a discount factor of $\gamma=0.99$ was considered. Time-slots $H=50$ and $T=50$ episodes were considered for wireless scheduling (synthetic) and 1-D bandit. Time-slots $H=500$ and $T=500$ episodes were considered for wireless scheduling (Real).   Moreover, we considered $V_n=V$ for all $n$ in the wireless scheduling (synthetic) and 1-D bandit problems. But, in wireless scheduling (Real), $V_n$ can vary across $n$ and depends on the dataset. In Figure \ref{fig:result}, we used $N=20, M=4$ for wireless scheduling (synthetic) and 1-D bandit. For wireless scheduling (Real), we consider $N=6$ and $M=1$. 

Table \ref{tab:performance} presents a comparative analysis of the cumulative regret, denoted as $\mathrm{Reg(T)}$, across three distinct experimental environments: a one-dimensional monotonic bandit, synthetic wireless scheduling, and real-world wireless scheduling. The performance of the proposed policy is evaluated against four baselines—UCWhittle, UCWhittle+Win, Random, and WIQL—across various system scales $(N,M)$. Across all applications and parameter configurations, the proposed policy consistently achieves the lowest cumulative regret. While UCWhittle+Win generally improves upon the standard UCWhittle by incorporating a sliding window for non-stationary environments, both are substantially outperformed by our method. This suggests that simply implementing a windowed version of UCWhittle is insufficient to guarantee optimal performance; rather, the window mechanism must be carefully designed to effectively capture the system's underlying dynamics. Finally, the Random and WIQL policies consistently exhibit the highest regret, with their performance gaps widening significantly as the number of resources M increases.

Figure \ref{fig:result} illustrates the cumulative regret, $\mathrm{Reg(T)}$, as a function of the number of episodes across the three experimental setups: (a) Synthetic Scheduling, (b) Real-Data Scheduling, and (c) 1-D Bandit. The plots are presented on a semi-logarithmic scale. Across all three applications, our Policy consistently achieves the lowest cumulative regret. UCWhittle and UCWhittle+Window algorithms are very close to each other, with the windowed version showing slightly better performance. However, both exhibit a significantly steeper regret slope compared to our proposed method. 


\section{Conclusions and Future Work}\label{conclusion}
This paper introduced an online/reinforcement learning algorithm for estimating the Whittle index for restless bandit problems with unknown and non-stationary transition kernels using sliding window and upper confidence bound approaches. To our knowledge, this is the first work to provide an upper bound of the dynamic regret of an online Whittle index-based algorithm for RMABs with unknown and non-stationary transition kernels. Our proposed algorithm is evaluated on two different restless bandit problems against four baselines and provides significant performance gains. We also provide novel regret analysis. An interesting direction of future work involves proving lower bounds for regret. Other future directions include extending this work to infinite or continuous state spaces, and designing algorithms that achieve sub-linear dynamic regret even for large $V_n$ (rapidly varying kernels).



%% file: appendix.tex
\subsection{Proof of Lemma 1}\label{plemma1}

The L1-deviation of the true distribution and the empirical distribution of $m$ events is bounded by \cite{weissman2003inequalities}:
\begin{align}\label{L1inequality}
    \mathrm{Pr}(|\hat p-p|_1 \geq \beta)\leq (2^m-2)\exp^{(-\frac{k \beta^2}{2})},
\end{align}
where $k$ is the number of samples.

    We denote $\mathbf 1(s',s,a,n,t)$ as an indicator variable that represents the event of state $s$, action $a$, and next state $s'$ for arm $n$ at one time slot of episode $t$. Similarly, $\mathbf 1(s',s,a,n,t,w)$ is an indicator variable that represents the event of state $s$, action $a$, and next state $s'$ for arm $n$ at one time slot in any one of the episodes $t-w+1, t-(w-1)+1, \ldots, t-1$. 

    By using \eqref{L1inequality} with $$\beta=\sqrt{\frac{2|\mathcal S|\mathrm{log}(2|\mathcal S||\mathcal A|NT/\eta)}{C^{(n)}_{t, w}(s,a)}}$$ and 
    $$k=C^{(n)}_{t, w}(s,a),$$ we get 
      \begin{align}\label{lemma2eq1}
        &\mathrm{Pr}\left(\|\hat P_{n, t, w}(\cdot|s, a)-\mathbb E[\mathbf 1(\cdot,s,a,n,t, w)]\|_1 \geq \sqrt{\frac{2|\mathcal S|\mathrm{log}(2|\mathcal S||\mathcal A|NT/\eta)}{C^{(n)}_{t, w}(s,a)}}\right)\nonumber\\&~~~~~~~~~~~~~~~~~~~~~~~~~~~~~~~~~~~~~~~~~~~~~~~~~~~~~~~~~~~~~~~~~~~~~~~~~~~~~~~~~~~~~~~~~~~~\leq \frac{\eta}{N|\mathcal S||\mathcal A|T}.
    \end{align}
    
    With probability one,  we have
    \begin{align}\label{lemma2eq2}
        &\|P_{n,t}(\cdot|s,a)-\mathbb E[\mathbf 1(\cdot,s,a,n,t, w)]\|_1\nonumber\\
        \leq &\|P_{n,t}(\cdot|s,a)-\max_{t'\in\{t-w+1, t-(w-1)+1, \ldots, t\}}\mathbb E[\mathbf 1(\cdot,s,a,n,t', 1)]\|_1\nonumber\\
        =&\|P_{n,t}(\cdot|s,a)-\max_{t'\in\{t-w+1, t-(w-1)+1, \ldots, t\}}P_{n,t'}(\cdot|s,a)]\|_1\leq w_nV_n/T.
    \end{align}

Now, by combining \eqref{lemma2eq1} and \eqref{lemma2eq2}, we have 
\begin{align}
    \mathrm{Pr}( P_{n,t} \in B_{n,t}, \forall n, \forall t) &\geq 1-\sum_{t=1}^{T}\sum_{n=1}^N\sum_{(s, a)\in \mathcal S \times A}\frac{\eta}{N |\mathcal S||\mathcal A|T}\nonumber\\
    &=1-\eta.
\end{align}
This concludes the proof of Lemma \ref{lemma1}.

\subsection{Proof of Theorem \ref{theorem1}}\label{Ptheorem1}

We focus on bounding the dynamic regret when the confidence bound  holds.
\begin{align}
    \mathrm{Reg}^{(t)}
    &=\sum_{t=1}^T\mathcal L(\pi^*_t, (P_{n, t})_{n=1}^N, \lambda^*_t)-\mathcal L(\pi_t, (P_{n, t})_{n=1}^N, \lambda_t)\nonumber\\
    &\leq \sum_{t=1}^T\mathcal L(\pi^*_t, (P_{n, t})_{n=1}^N, \lambda_t)-\mathcal L(\pi_t, (P_{n, t})_{n=1}^N, \lambda_t)\nonumber\\
    &=\sum_{t=1}^T\sum_{n=1}^NU(\pi_{n, t}^*, P_{n,t}, \lambda_t)-U(\pi_{n, t}, P_{n,t}, \lambda_t)\nonumber\\
    &\leq \sum_{t=1}^T\sum_{n=1}^NU(\pi_{n, t}, \tilde P_{n,t}, \lambda_t)-U(\pi_{n, t}, P_{n,t}, \lambda_t), 
 \end{align}    
where the first inequality holds because $\lambda^*_t$ minimizes $\mathcal L(\pi^*_t, (P_{n, t})_{n=1}^N, \lambda)$ for all $
\lambda\geq 0$ and the second inequality holds because of \eqref{optimistic1}. Then, by applying value difference theorem \cite[Theorem 6.4] {wang2023optimistic}, we have 

\begin{align}
    &U(\pi_{n, t}, \tilde P_{n,t}, \lambda_t)-U(\pi_{n, t}, P_{n,t}, \lambda_t)\nonumber\\
    &\!\leq \mathbb E_{P_{n, t}, \pi_{n, t}} \bigg[ \sum_{(s,a)\in \mathcal S} \alpha_{t}^{(n)}(s, a) \sum_{s'\in \mathcal S}\bigg|\tilde P_{n, t}(s'|s, a)-P_{n, t}(s'|s, a)\bigg|\bigg]V_{max}\nonumber\\
    &=\mathbb E_{P_{n, t}, \pi_{n, t}} \bigg[ \sum_{(s,a)\in \mathcal S} \alpha_{t}^{(n)}(s, a) \sum_{s'\in \mathcal S}d_{t}^{(n)}(s,a)\bigg]V_{max}
\end{align}
where $V_{max}=\max_{n \in [N], s \in \mathcal S} V_n(s'; \pi_{n,t}, P_{n,t})$. After this, by substituting the value of $d_{t}^{(n)}(s,a)$, we get 
\begin{align}
    &\sum_{t=1}^T \sum_{n=1}^N \mathbb E_{P_{n, t}, \pi_{n, t}}\bigg[ \sum_{(s,a)\in \mathcal S \times \mathcal A} \alpha_{t}^{(n)}(s, a) d_{t}^{(n)}(s, a)\bigg]\nonumber\\
    &\leq \sum_{t=1}^T\bigg( \sum_{n=1}^N \sqrt{2|\mathcal S|\mathrm{log}(2|\mathcal S||\mathcal A|NT/\eta)}  \mathbb E_{P_{n, t}, \pi_{n, t}}\bigg[ \sum_{(s, a) \in \mathcal Z_2} \frac{\alpha_{t}^{(n)}(s, a)}{\sqrt{C^{(n)}_{t,w_n}(s, a)}}\bigg]\nonumber\\&+w_nV_n/T H\bigg)\nonumber\\
    &=\sum_{t=1}^T\bigg(\sum_{n=1}^N\sqrt{2|\mathcal S|\mathrm{log}(2|\mathcal S|| \mathcal A|NT/\eta)}   \sum_{(s, a) \in \mathcal S \times \mathcal A}g_{t,n}(s,a,w_n)+w_nV_n H/T\bigg)\nonumber\\
    &\leq \sum_{t=1}^T\bigg(\sum_{n=1}^N \sqrt{2|\mathcal S|\mathrm{log}(2|\mathcal S||\mathcal A|NT/\eta)} 2|\mathcal S|G_{t,n}(w_n)+w_nV_nH/T\bigg)\nonumber\\
    &=\sum_{t=1}^T O\bigg(\sum_{n=1}^N 2|\mathcal S| G_{t,n}(w_n)+w_nV_nH/T\bigg),
\end{align}
where $G_{t,n}(w)=\max_{(s,a)\in \mathcal S\times \mathcal A} g_{t,n}(s,a,w)$ and $$g_{t,n}(s,a,w)=\mathbb E_{P_{n, t}, \pi_{n, t}}\bigg[ \sum_{(s, a) \in \mathcal Z_2} \frac{\alpha_{t}^{(n)}(s, a)}{\sqrt{C^{(n)}_{t,w}(s, a)}}\bigg]$$ is a non-increasing function of the window size $w$. This is because $C^{(n)}_{t,w}(s, a)$ is a non-decreasing function of $w$.

\subsection{Proof of Theorem \ref{theorem2}}\label{Ptheorem2}

Lets denote the probability to visit every $(s,a) \in \mathcal S\times \mathcal A$ at least once in an episode for all arms $n\in [N]$ by $P_{\mathrm{min}}$. According to the condition in Theorem \ref{theorem2}, $P_{\mathrm{min}}>0$.

Now, to prove Theorem \ref{theorem2}, we bound 
\begin{align}
   g_{t,n}(s, a, w)&= \mathbb E_{P_{n, t}, \pi_{n, t}}\left[ \frac{\alpha_{t}^{(n)}(s, a)}{\sqrt{C^{(n)}_{t,w}(s, a)}}\right] \leq H \mathbb E_{P_{n, t}, \pi_{n, t}}\left[ \frac{1}{\sqrt{C^{(n)}_{t,w}(s, a)}}\right],
\end{align}
where the number of visit $\alpha_{t}^{(n)}(s, a)$ in one episode is upper bounded by the time horizon $H$.

Let $\mathbb E[C^{(n)}_{t,w}(s, a)]=\mu$. Then, $\mu\geq 1$ because by definition, $C^{(n)}_{t, w}(s, a):=\max\bigg\{\sum_{s^{\prime} \in \mathcal S} C_{t, w}^{(n)}(s', a, s), 1\bigg\}.$ Moreover, $\mu\geq w_n P_{\mathrm{min}}$.

Now, we have 
\begin{align}
    \mathbb E\left[ \frac{1}{\sqrt{C^{(n)}_{t,w}(s, a)}}\right] &= \mathbb E\left[ \frac{1}{\sqrt{C^{(n)}_{t,w}(s, a)}} \bigg| C^{(n)}_{t,w}(s, a) < \frac{\mu}{2}\right] P\left(C^{(n)}_{t,w}(s, a) < \frac{\mu}{2}\right) \nonumber\\ &+ \mathbb E\left[ \frac{1}{\sqrt{C^{(n)}_{t,w}(s, a)}} \bigg| C^{(n)}_{t,w}(s, a) \geq \frac{\mu}{2}\right] P\left(C^{(n)}_{t,w}(s, a) \geq \frac{\mu}{2}\right)
\end{align}

If $C^{(n)}_{t,w}(s, a) \geq \mu/2$, then $1/\sqrt{C^{(n)}_{t,w}(s, a)} \leq 1/\sqrt{\mu/2} = \sqrt{2/\mu}.$
This part of the expectation is therefore bounded by $\sqrt{2/\mu} \cdot P(C^{(n)}_{t,w}(s, a) \geq \mu/2) \leq \sqrt{2/\mu} \leq \sqrt{\frac{2}{w_n \mathrm p_{min}}}.$

If $C^{(n)}_{t,w}(s, a) < \mu/2$, there exists a constant $\eta>0$ such that we have $P(C^{(n)}_{t,w}(s, a) < (1 - 1/2)\mu) \leq O(e^{-\mu/4\eta})\leq O(e^{-w_n P_{\mathrm{min}}/4\eta})$ by using the Chernoff bound for Markov Chains \cite{chung2012chernoff}. 

Thus, the expectation becomes
\begin{align}
    \mathbb E\left[ \frac{1}{\sqrt{C^{(n)}_{t,w}(s, a)}}\right] \leq  \frac{\sqrt2}{\sqrt{w_n P_{\mathrm{min}}}}+O(e^{-w_n P_{\mathrm{min}}/4\eta}).
\end{align}
We can have a constant $\eta_1>0$ independent of $w_n$ and $P_{\mathrm{min}}$ such that 
\begin{align}
   \frac{\sqrt2}{\sqrt{w_n P_{\mathrm{min}}}}+e^{-w_n P_{\mathrm{min}}/4\eta} \leq \frac{\eta_1}{\sqrt{w_n P_{\mathrm{min}}}}=O(1/\sqrt{w_n P_{\mathrm{min}}}).
\end{align}

Therefore, we have 
\begin{align}
   g_{t,n}(s, a, w)&= \mathbb E_{P_{n, t}, \pi_{n, t}}\left[ \frac{\alpha_{t}^{(n)}(s, a)}{\sqrt{C^{(n)}_{t,w}(s, a)}}\right] \leq O\bigg(\frac{H}{\sqrt{w_n P_{\mathrm{min}}}}\bigg)
\end{align}

Next, we have 
\begin{align}
    &\sum_{t=1}^T \bigg(\frac{H}{\sqrt{w_n P_{\mathrm{min}}}}+ \frac{w_n V_nH}{T}\bigg)=H\bigg(\frac{T}{\sqrt{w_nP_{\mathrm{min}}}}+V_{n}w_n\bigg).
\end{align}

Then, by substituting $w_n=(T/V_n)^{2/3}$, we get
\begin{align}
    &H\bigg(\frac{T}{\sqrt{w_nP_{min}}}+w_n V_{n}\bigg)=HT^{2/3}V_{n}^{1/3}P_{\mathrm{min}}^{-1/2}+HT^{2/3}V_{n}^{1/3}=\tilde O\bigg(T^{2/3}V_{n}^{1/3}\bigg),
\end{align}
where $H$ and $P_{\mathrm{min}}$ are absorbed in the big-O-notation because $H$ is a constant number of time-slots in every episode, $P_{\mathrm{min}}$ depends on the number of time-slots $H$ and the initial state in any episode.

\subsection{Proof of Theorem \ref{theorem4}}\label{Ptheorem4}
We first decompose the upper bound of $\mathrm{Reg(T)}$ into two separate terms. We solve another bandit problem to select $V_n$ from a set of possible drift values based on the history by using EXP3 algorithm \cite{auer2002nonstochastic}. In this case, we can decompose the regret associated with $\mathrm{Reg(T)}$ as follows: 
\begin{align}
\mathrm{Reg(T)}&\leq \sum_{t=1}^T\sum_{n=1}^N U(\pi^*_{n,t}, P_{n, t}, \lambda)-U(\pi_{n,t}(V_n(i_{n,t})), P_{n, t}, \lambda) 
\nonumber\\
&=\sum_{t=1}^T \sum_{n=1}^N\bigg( U(\pi^*_{n,t}, P_{n, t}, \lambda)-U(\pi_{n,t}(V_n^*), P_{n, t}, \lambda) \bigg) \nonumber\\
&~~+ \sum_{t=1}^T \sum_{n=1}^N\bigg( U(\pi_{n,t}(V_n^*), P_{n, t}, \lambda)-U(\pi_{n,t}(V_n(i_{n,t})), P_{n, t}, \lambda) \bigg)
\end{align}
where we denote $V_n^*$ is the optimal choice from the finite set of possible budgets, $V_n(i_{n,t})$ is estimated, and $\pi_n(V_n)$ denotes our policy when $V_n$ is used. 

The optimal choice of $V_n^*$ satisfies:  
\begin{align}
&\max_{(s, a)\in \mathcal S \times A} \sum_{s'\in \mathcal S} \bigg|P_{n, t}(s'|s, a)- P_{n, t-1}(s'|s, a)\bigg|\nonumber\\
&\leq V_n^*/T\nonumber\\ 
&\leq \max_{(s, a)\in \mathcal S \times A} \sum_{s'\in \mathcal S} \bigg|P_{n, t}(s'|s, a)- P_{n, t-1}(s'|s, a)\bigg|+2/J_nT.
\end{align}
Consequently, the optimal upper bound of the total variation budget $V_{n}$ satisfies the following:
$$V_{n}\leq V_{n}^*\leq V_{n}+2T/J_n.$$ 

By substituting $V_n=V_{n}^*+2T/J_n$ in Theorem \ref{theorem2}, we can obtain:
\begin{align}
    &\sum_{t=1}^T \sum_{n=1}^N\bigg( U(\pi^*_{n,t}, P_{n, t}, \lambda)-U(\pi_{n,t}(V_n^*), P_{n, t}, \lambda) \bigg)\nonumber\\
    &\leq \sum_{n=1}^N\tilde O(T^{2/3}( V_n+2T/J_n)^{1/3}).
\end{align}

For the other term, we can directly use the regret upper bound of EXP3 and get: 
\begin{align}
    \sum_{t=1}^T \sum_{n=1}^N\bigg( U(\pi_{n,t}(V_n^*), P_{n, t}, \lambda)-U(\pi_{n,t}(V_n(i_{n,t})), P_{n, t}, \lambda) \bigg)\leq \tilde O(\sqrt{TJ_n}).
\end{align}


